\documentclass[acmsmall]{acmart}
\usepackage{algorithm}
\usepackage{algorithmic}
\usepackage{booktabs}
\usepackage{multirow}

\AtBeginDocument{%
  \providecommand\BibTeX{{%
    \normalfont B\kern-0.5em{\scshape i\kern-0.25em b}\kern-0.8em\TeX}}}

\setcopyright{none}
\renewcommand\footnotetextcopyrightpermission[1]{}

\begin{document}

\title{Accelerating Floating-Point Satisfiability Solving via Gradient Normalization}

\author{Yuanzhuo Zhang}
\affiliation{%
  \institution{Independent}
  \country{China}
}

\begin{abstract}
Satisfiability Modulo Theories (SMT) solvers are foundational to software verification, program analysis, and compiler testing, particularly over the theory of Quantifier-Free Floating-Point ({QF\_FP}). While recent optimization-based SMT solvers have successfully applied gradient descent to continuous relaxations of logical formulas, they are fundamentally bottlenecked by gradient domination—a phenomenon where a small subset of difficult clauses hijacks the optimization trajectory, preventing the solver from satisfying the broader formula and trapping it in local minima.

To overcome this, we present GradSAT, a novel framework that bridges optimization-based SMT solving with Multi-Task Learning (MTL). GradSAT reformulates the constraint satisfaction process by treating each SMT clause as an independent MTL task. By applying dynamic gradient normalization (GradNorm), GradSAT actively balances the gradient magnitudes across all clauses at runtime, systematically penalizing dominant gradients and accelerating lagging clauses to ensure uniform convergence. GradSAT implements this through a highly optimized, two-stage hybrid pipeline. First, a GPU-accelerated PyTorch backend—leveraging symbolic compilation and operator fusion—navigates the continuous relaxation to a high-quality basin. Second, the candidate assignment is handed off to a bit-precise local search engine to rapidly resolve the exact, rigorous assignment. By stabilizing the continuous search dynamics, GradSAT mitigates the brittleness of prior gradient-based solvers and provides a robust, highly parallelizable architecture for complex constraint solving.
\end{abstract}

\maketitle

\section{Introduction}
\label{sec:intro}

Satisfiability Modulo Theories (SMT) solvers constitute a foundational layer of modern software engineering infrastructure. They serve as the reasoning engines behind software verification, symbolic execution, bounded model checking, compiler testing, automated program repair, program synthesis, and a broad range of static and dynamic analysis pipelines. In these systems, the solver is not an isolated backend component whose performance can be amortized away; rather, its scalability and robustness often determine the practical reach of the entire tool. A verification engine that issues thousands of solver queries, a symbolic executor that repeatedly solves path conditions, or a compiler-testing framework that depends on semantic equivalence checks is only as effective as the underlying reasoning procedure permits. Consequently, improving SMT solving is not merely a logic-level concern, but a systems-level concern for software engineering.

Within the SMT landscape, quantifier-free floating-point arithmetic (QF\_FP) is especially important and especially difficult. Floating-point computation pervades performance-critical and correctness-critical software, including numerical libraries, scientific software, graphics and vision kernels, machine learning runtimes, embedded control software, cyber-physical systems, and safety-critical code in aerospace, automotive, and finance. For software engineering tools that aim to reason faithfully about real programs, QF\_FP is indispensable because it captures the precise semantics of IEEE-754 arithmetic, including rounding modes, exceptional values, subnormals, overflow, underflow, cancellation, and non-associativity. These behaviors are not peripheral edge cases. They are central to whether a compiler transformation is semantics-preserving, whether a symbolic execution engine reports a feasible path correctly, and whether a verification tool can establish or refute a property of a numerically intensive program.

At the same time, floating-point reasoning is notoriously challenging. Exact SMT solvers such as Z3, cvc5, and MathSAT typically handle QF\_FP by reducing floating-point operators to bit-precise encodings via eager or lazy bit-blasting, after which a SAT-level search procedure carries most of the burden. This approach is semantically rigorous and has enabled major progress in automated reasoning. However, it often scales poorly for formulas containing deeply nested floating-point expressions, non-linear arithmetic structure, multiple interacting rounding behaviors, or large families of path constraints generated by software engineering workloads. Multiplication, division, fused operations, and comparisons over special values can generate large and dense Boolean encodings, causing severe memory pressure and combinatorial blowups. In practice, such costs are directly visible at the tool level: verification conditions time out, symbolic execution loses depth, test generation misses cases, and compiler-analysis pipelines become too expensive to run at scale.

These limitations have motivated the search for alternative solving paradigms. One line of work explores incomplete but lightweight search-based methods, such as stochastic local search and fuzzing-inspired techniques, which attempt to discover satisfying assignments through mutation, coverage guidance, or meta-heuristic search. These approaches can be effective on some satisfiable benchmarks, particularly when a witness is relatively easy to stumble upon. However, they lack stable directional information and often struggle with formulas whose satisfying regions are narrow, whose arithmetic structure is highly non-linear, or whose Boolean organization requires coordinated progress across many constraints. For software engineering applications that value robustness across diverse workloads, such behavior can be difficult to rely upon.

A second and increasingly promising line of work reimagines satisfiability as continuous optimization. Systems such as goSAT, XSat, and Grater demonstrate that SMT solving can be relaxed into the minimization of a differentiable objective function over continuous variables. In this paradigm, the discrete search for a model is replaced by gradient-based descent over a smooth surrogate landscape. This shift is attractive for software engineering workloads because it offers a way to exploit modern optimization infrastructure, including automatic differentiation and hardware acceleration, while potentially traversing search spaces that are hostile to purely discrete search. In particular, continuous optimization may better exploit arithmetic structure that naturally arises in floating-point-heavy verification problems.

Despite this promise, existing optimization-based SMT solvers remain brittle on realistic formulas. A central issue is that the clauses of a formula rarely contribute equally to the optimization dynamics. Real SMT instances produced by software engineering tools are structurally heterogeneous: some constraints are nearly linear and easy to satisfy, while others are highly non-linear, involve transcendental or composite floating-point structure, or induce rapidly changing local curvature. When such clauses are aggregated into a single global objective using naive equal weighting, the resulting gradient field is often dominated by a small subset of numerically volatile constraints. Those constraints can produce disproportionately large gradients and thereby hijack the descent trajectory. As a result, the optimizer over-focuses on a few difficult clauses, makes little progress on the remainder of the formula, and frequently converges to poor local minima or oscillatory regions. This phenomenon is especially problematic for software engineering workloads, where solver queries are often large, unbalanced, and generated automatically rather than hand-designed to be optimization-friendly.

We identify this pathology as a form of \emph{gradient domination}. Importantly, it is not a superficial tuning issue but a structural mismatch between the heterogeneous nature of SMT formulas and the homogeneous way most continuous solvers aggregate clause losses. In realistic workloads, the difference in gradient scale across clauses can be substantial, and those scales evolve dynamically as the search moves through the relaxed space. A static weighting scheme therefore provides little protection. To make continuous SMT solving robust enough for practical software engineering use, the optimization process must adapt online to clause-level imbalance.

In this paper, we address this challenge by drawing a novel connection between optimization-based SMT solving and multi-task learning (MTL). In MTL, a shared model is trained to optimize multiple tasks simultaneously, and a well-known challenge is that tasks with larger or faster-changing gradients can dominate the shared parameters, slowing or even preventing balanced learning across the task set. This challenge is conceptually strikingly similar to the clause imbalance problem described above. Our key observation is that a quantifier-free SMT formula admits a natural MTL interpretation: each top-level clause may be treated as an independent task, while the relaxed SMT variables act as shared parameters optimized jointly across all tasks.

This perspective immediately suggests that techniques developed for balancing gradients in MTL may be repurposed for SMT solving. In particular, we adapt GradNorm, a dynamic gradient normalization technique originally developed for neural multi-task optimization, to the setting of clause-level SMT relaxation. Rather than allowing a few dominant clauses to monopolize the descent direction, our method continuously monitors per-clause gradient norms and adjusts clause weights to encourage more balanced progress across the formula. Intuitively, clauses that are lagging behind receive restorative emphasis, while clauses whose gradients are disproportionately large are prevented from overwhelming the global search. This does not eliminate the inherent difficulty of QF\_FP solving, but it changes the optimization dynamics in a way that is markedly more stable on heterogeneous formulas.

Based on this idea, we present \textbf{GradSAT}, a gradient-normalized optimization framework for SMT solving over QF\_FP. GradSAT is designed with software engineering deployment scenarios in mind. It does not aim to replace exact solvers outright, nor does it rely solely on an incomplete continuous search. Instead, it adopts a two-stage hybrid architecture. In the first stage, GradSAT compiles the relaxed clause objectives into a GPU-executable PyTorch computation graph and performs gradient-based search with dynamic clause-level normalization. This stage is intended to navigate quickly toward a promising region of the search space by exploiting arithmetic structure and hardware parallelism. In the second stage, GradSAT hands the candidate solution to a bit-precise local refinement engine, which restores exact floating-point semantics and validates the final assignment. This hybrid design reflects a core practical requirement of software engineering tools: scalability is valuable, but semantic fidelity remains essential.

The architectural choices in GradSAT are also motivated by an engineering view of solver design. Existing continuous solvers are often limited not only by optimization instability but also by execution overhead. Formula evaluation through scalarized CPU-based autograd pipelines can become a major bottleneck on large SMT instances. By contrast, GradSAT treats clause evaluation and differentiation as batched tensor computation and leverages GPU parallelism together with compilation and operator fusion to reduce this cost. This makes the approach relevant not only as a conceptual reformulation, but also as a systems contribution aimed at the practical throughput demands of solver-intensive software engineering pipelines.

Our broader thesis is that SMT solver design for software engineering should be viewed as an infrastructure problem at the intersection of formal methods, optimization, and systems. From this perspective, techniques from modern machine learning should not be seen only as heuristic predictors layered on top of legacy solvers, but also as a source of optimization principles and execution infrastructure that can directly reshape how SMT search is carried out. GradSAT instantiates this thesis by combining (i) a new MTL-based interpretation of clause optimization, (ii) dynamic gradient balancing to counteract clause-level interference, and (iii) a GPU-oriented execution strategy that aligns with the scale and diversity of practical QF\_FP workloads.

In summary, this paper makes the following contributions:
\begin{itemize}
    \item \textbf{SE-oriented reframing of continuous SMT solving.} We position optimization-based QF\_FP solving as a practical infrastructure problem for software verification, symbolic execution, compiler testing, and related software engineering workflows, where solver robustness directly impacts end-to-end tool effectiveness.
    \item \textbf{SMT as Multi-Task Learning.} We formalize a novel correspondence between clause-level SMT optimization and multi-task learning, treating top-level clauses as tasks and relaxed SMT variables as shared parameters.
    \item \textbf{Dynamic Gradient Normalization for SMT.} We adapt GradNorm to continuous SMT solving, introducing an online mechanism that balances clause-level gradient magnitudes and mitigates gradient domination on heterogeneous formulas.
    \item \textbf{The GradSAT Architecture.} We present a two-stage solver framework that combines GPU-accelerated continuous search with bit-precise local refinement, thereby balancing optimization scalability with exact IEEE-754 semantics.
    \item \textbf{Empirical Evaluation.} We evaluate GradSAT on standard QF\_FP benchmark suites against exact and optimization-based baselines, demonstrating that gradient normalization improves robustness and makes continuous solving more effective on difficult floating-point instances.
\end{itemize}

\section{Overview}
\label{sec:overview}

Traditional Satisfiability Modulo Theories (SMT) solving fundamentally involves searching for variable assignments that satisfy a set of constraints. We propose a bridge between optimization-based SMT solving and machine learning. Specifically, we formulate the constraint satisfaction process through the lens of multi-task learning (MTL).

\subsection{SMT Solving as Multi-Task Learning}

To convert SMT solving to MTL, our framework maps the logical components of an SMT formula directly to an ML training paradigm. We achieve this through the following correspondence:
\begin{itemize}
  \item Treat each SMT clause as a separate MTL ``task.''
  \item Treat each clause variable $x$ as an MTL task parameter $\theta$.
  \item Minimize the total SMT residual $\mathcal{R}_{total}$ $\Rightarrow$ Minimize the MTL total loss $\mathcal{L}_{total}$.
\end{itemize}

To understand the power of this formulation, consider the structural differences between traditional machine learning and our SMT-MTL setting. In standard deep learning, $\theta$ represents the weights of a neural network mapping discrete inputs to predictions. In our framework, $\theta$ directly represents the continuous relaxations of the SMT variables themselves; there is no underlying network architecture. The ``forward pass'' consists entirely of evaluating the mathematical operations and logical comparators defined by the SMT formula under the current continuous assignment. 

This mapping allows us to leverage continuous optimization techniques, such as gradient descent, to simultaneously optimize the parameters across all clauses, effectively searching for a globally satisfying assignment. The key advantage is that we can now import techniques from the multi-task learning literature---in particular, gradient normalization---to handle the fundamental challenge of gradient domination, where a small number of hard clauses hijack the optimization and prevent the solver from satisfying the remaining clauses.

Viewing SMT through the MTL lens provides a rigorous framework for understanding why prior continuous solvers frequently fail. In standard continuous SMT systems, the total objective is implicitly formulated as a naive, unweighted sum of individual clause residuals: $\mathcal{L}_{total} = \sum_{i} \mathcal{L}_i$. However, mathematical constraints inherently exhibit wildly different gradient scales. For instance, the gradient of a high-degree polynomial floating-point constraint will scale non-linearly, whereas the gradient of a simple linear equality remains constant. Under an unweighted sum, the optimization landscape becomes pathological; the gradients from highly non-linear or mathematically dense clauses completely eclipse those from simpler logical constraints.

This phenomenon is strictly analogous to task imbalance in MTL, where a model overfits to a task with a massive loss magnitude while entirely ignoring the others. Furthermore, because the gradient scales of SMT clauses fluctuate drastically as the variables move through the continuous space, static weighting strategies (e.g., manually assigning lower weights to non-linear clauses) are insufficient. By formalizing the solver as an MTL system, we can directly apply adaptive gradient balancing techniques. This allows the solver to dynamically adjust the influence of each clause at every optimization step, mitigating the destructive interference of conflicting gradients and ensuring that all logical constraints descend uniformly toward satisfiability.

\subsection{Problem Setting: Floating-Point SMT}

We target satisfiability over the theory of floating-point arithmetic (QF\_FP), a fundamental and notoriously challenging problem in domains such as numerical software verification, cyber-physical systems analysis, and compiler optimization testing. Given a quantifier-free formula $\varphi$ over floating-point variables $x_1, \ldots, x_n$, the goal is to find an assignment $\sigma$ such that $\sigma \models \varphi$, or to determine that no such assignment exists.

Unlike integer or bit-vector arithmetic, floating-point logic introduces severe combinatorial explosions for traditional DPLL(T)-based solvers. This complexity arises from its highly non-linear semantics, dynamic rounding modes, and the required handling of special values (e.g., \texttt{NaN}, \texttt{+INF}, \texttt{-INF}). Exact solvers typically resort to bit-blasting these operations into massive propositional circuits, which frequently leads to insurmountable memory and time bottlenecks on complex benchmarks.

Optimization-based SMT solvers offer a scalable alternative by encoding the logical formula $\varphi$ into a continuous, differentiable objective function $\mathcal{R}_{total}(x)$. In this paradigm, the discrete floating-point variables are relaxed into real-valued continuous variables $x \in \mathbb{R}^n$. The semantics of floating-point operations, rounding behaviors, and comparators are approximated using smooth, differentiable mathematical functions. The solver then traverses this continuous space using gradient descent to find a global minimum.

When the objective function is minimized such that $\mathcal{R}_{total}(x) = 0$ (or falls below a strict precision threshold $\epsilon$), the corresponding continuous values are mapped back to their nearest valid floating-point representations. If this rounded assignment satisfies every clause in the original discrete formula $\varphi$, the solver terminates successfully. Prior continuous systems, such as XSat and Grater, have demonstrated that this relaxation-based gradient descent can be highly effective, sometimes outperforming mature exact solvers on hard, non-linear benchmarks. However, successfully mapping QF\_FP to this continuous space frequently creates heavily distorted loss landscapes with wildly varying gradient magnitudes across different clauses, directly motivating our gradient-normalized approach.

\subsection{End-to-End Pipeline}

GradSAT operates in a two-stage pipeline:

\paragraph{Stage 1: GPU-Accelerated Gradient Descent with GradNorm.} The input SMT-LIB formula is parsed via the Z3 API, symbolically simplified through SymPy rewrite rules, and compiled into a PyTorch computation graph whose operations are fully differentiable. A custom Adam-based gradient descent minimizer---equipped with per-clause GradNorm weight normalization---drives the relaxed variables toward a basin where $\mathcal{R}_{total}(x)$ is close to zero. By leveraging GPU parallelism and \texttt{torch.compile} for operator fusion, Stage 1 uses square residuals metrics and the GPU for fast convergence to a good basin.

\paragraph{Stage 2: Bit-Precise Local Search.} Once Stage 1 converges to a promising basin, GradSAT hands the candidate assignment to StageSAT's S3 local-search engine, which refines the solution with bit-precise semantics and fine-grained neighborhood search to ensure bit-precise and fine-grained search. Starting from a high-quality initial point dramatically reduces the search space compared to a random initialization.

\section{Method}
\label{sec:method}

\subsection{Objective Formulation}

In a standard continuous relaxation of SMT solving, the total objective $\mathcal{R}_{total}$ is computed over the input variables $x$, where $x$ are the input variables and $d(\ldots)$ is the base loss function for one clause (based on the comparator and the left and right sides of a comparison). The formula is represented as:
\begin{equation}
\mathcal{R}_{total}(x)=\sum_{j=1}^{|J|}\prod_{i=1}^{|I_{j}|}d(\bowtie_{i,j},\, e_{i,j}(x),\, e_{i,j}^{\prime}(x))
\end{equation}

To solve the SMT problem, our goal is to minimize $\mathcal{R}_{total}$, which translates to minimizing the MTL equivalent $\mathcal{L}_{total}$. The multi-task learning objective function is formulated as the weighted sum of individual task losses:
\begin{equation}
\mathcal{L}_{total}=\sum_{i=1}^{T}w_{i}\mathcal{L}_{i}(\theta_{i})
\end{equation}
where $\mathcal{L}_{i}$ is the loss for task $i$, $w_{i}$ is the weight for task $i$, and $\theta_{i}$ represents the parameters for task $i$.

\subsection{Differentiable Clause Encoding}

Each top-level conjunct $j$ ($1 \le j \le |J|$) is a disjunction of atomic predicates $p_{i,j}$ ($1 \le i \le |I_j|$), where each predicate has the form $e_{i,j}(x) \bowtie_{i,j} e'_{i,j}(x)$ with $\bowtie \in \{=, \neq, <, \le, >, \ge\}$.

We define a smooth, non-negative base loss $d$ for each comparator:
\begin{align}
  d({=},\, e,\, e')    &= (e - e')^2 \\
  d({\le},\, e,\, e')  &= \max(0,\; e - e')^2 \\
  d({\neq},\, e,\, e') &= \exp\!\bigl(-\gamma\,(e - e')^2\bigr)
\end{align}
where $\gamma > 0$ is a scaling hyperparameter. Other comparators are derived by symmetry. Each base loss satisfies $d = 0$ if and only if the corresponding predicate holds under the real-valued relaxation.

For a clause $j$ that is a disjunction $p_{1,j} \vee \cdots \vee p_{|I_j|,j}$, the clause is satisfied when at least one predicate holds. We encode this via a product:
\begin{equation}
  L_j(x) = \prod_{i=1}^{|I_j|}
    d\!\bigl(\bowtie_{i,j},\; e_{i,j}(x),\; e'_{i,j}(x)\bigr)
\end{equation}
The product is zero whenever any factor is zero (i.e., any disjunct is satisfied), correctly encoding disjunctive semantics.

The total objective is then a weighted sum over all clauses:
\begin{equation}
  \mathcal{L}_{total}(x;\, w)
  = \sum_{j=1}^{|J|} w_j(t)\; L_j(x)
\end{equation}
where $w_j(t)$ are trainable, time-varying weights managed by the GradNorm procedure described below.

\subsection{Dynamic Weight Normalization}

A naive equal-weighting scheme ($w_i = 1$) often leads to suboptimal training dynamics where tasks with larger loss scales dominate the training. To address this, we implement dynamic weight normalization designed by GradNorm. The core idea is to make each task training more balanced by maintaining relatively similar descent speeds across all tasks. Compared to using equal weights, weight normalization should descend faster and may find a smaller global minimum.

\paragraph{Gradient Norms.}
At training time $t$, we define the L2 norm of the gradient of the weighted single-task loss (the real gradient norm for task $i$) with respect to the chosen weights $W$ as:
\begin{equation}
G_{W}^{(i)}(t)=\|\nabla_{W}(w_{i}(t)L_{i}(t))\|_{2}
\end{equation}
which places gradient norms on a common scale. The average gradient norm across all tasks is then given by:
\begin{equation}
\overline{G}_{W}(t)=E_{\text{task}}[G_{W}^{(i)}(t)] = \frac{1}{|J|}\sum_{j=1}^{|J|} G^{(j)}_W(t)
\end{equation}

\paragraph{Inverse Training Rates.}
To balance the training rates, we compute the inverse training rate for each task. If $r$ is big, it means that the learning rate is slow. The inverse training rate is defined as:
\begin{equation}
\tilde{L}_{i}(t)=\frac{L_{i}(t)}{L_{i}(0)}
\end{equation}
This allows us to calculate the relative inverse training rate for each task:
\begin{equation}
r_{i}(t)=\frac{\tilde{L}_{i}(t)}{\mathbb{E}_{\text{task}}[\tilde{L}_{i}(t)]}
\end{equation}
A higher value of $r_{i}(t)$ dictates that the gradient magnitudes should be higher for task $i$ to encourage it to train more quickly.

\paragraph{Target Gradient Norms.}
The objective gradient norm for each task $i$ maps to:
\begin{equation}
G_{W}^{(i)}(t) \mapsto \overline{G}_{W}(t)\,[r_{i}(t)]^{\alpha}
\end{equation}
where $\alpha$ acts as an asymmetry hyperparameter that sets the strength of the restoring force which pulls tasks back to a common training rate. When tasks are very different in their complexity, a higher value of $\alpha$ enforces stronger training rate balancing. When $\alpha = 0$, the algorithm simply tries to equalize all gradient norms.

\paragraph{GradNorm Loss.}
We compute the loss function $L_{grad}$ between the actual and target gradient norms to update the weights:
\begin{equation}
L_{grad}(t;w_{i}(t))=\sum_{i}\|G_{W}^{(i)}(t)-\overline{G}_{W}(t)[r_{i}(t)]^{\alpha}\|_{1}
\end{equation}
When differentiating this loss $L_{grad}$, we treat the target gradient norms $\overline{G}_{W}(t) \times [r_{i}(t)]^{\alpha}$ as a fixed constant to prevent divergences to zero. This loss $L_{grad}$ is then differentiated only with respect to the $w_{i}$ parameters, as the $w_i(t)$ directly control gradient magnitudes per task. After every update step, we renormalize the weights $w_i(t)$ so that $\sum_i w_i(t) = T$ in order to decouple gradient normalization from the global learning rate.

\subsection{Adaptation to the SMT Setting}

While GradNorm was originally designed for deep multitask networks with shared convolutional layers, our SMT setting differs in several important ways:

\paragraph{Number of tasks.} Typical MTL setups involve 2--10 tasks. In SMT solving, the number of clauses $|J|$ can range from tens to thousands. We handle this by computing gradient norms in batch on the GPU and using the clause-averaged norm $\overline{G}_{W}(t)$ as the reference scale.

\paragraph{Shared layer selection.} In neural MTL, $W$ is typically the last shared layer. In our setting, all variables $x$ are shared by all clauses, so we set $W = x$.

\paragraph{Task symmetry.} Clauses derived from the same SMT formula often have widely varying structural complexity (e.g., linear vs.\ nonlinear arithmetic). We find that moderate values of $\alpha$ (between 0.5 and 1.5) provide robust performance across diverse benchmarks.

\paragraph{Convergence criterion.} Unlike MTL where training runs for a fixed number of epochs, our solver terminates when $\mathcal{R}_{total}(x) = 0$ (or falls below a precision threshold $\epsilon$). The weight normalization must therefore adapt quickly to changing clause dynamics, motivating a relatively high learning rate $\eta_w$ for the weight optimizer.

\subsection{Complete Algorithm}

Algorithm~\ref{alg:GradSAT} summarizes the complete Stage~1 optimization procedure.

\begin{algorithm}[htbp]
\caption{GradSAT Stage~1: Gradient-Normalized SMT Optimization}
\label{alg:GradSAT}
\begin{algorithmic}[1]
\REQUIRE SMT formula $\varphi$ with $|J|$ clauses, hyperparameters
         $\alpha$, $\eta$, $\eta_w$, precision $\epsilon$, max steps $T$
\STATE Parse $\varphi$ via Z3 API, simplify with SymPy rules
\STATE Compile clause losses $L_1, \ldots, L_{|J|}$ into PyTorch graph
\STATE Initialize variables $x \sim \mathcal{U}(-1, 1)^n$
\STATE Initialize weights $w_j \leftarrow 1$ for all $j$
\STATE Record initial losses $L_j(0)$ for all $j$
\FOR{$t = 0$ \TO $T$}
  \STATE \textbf{Forward:} compute $L_j(x)$ for all $j$ and
         $\mathcal{L}_{total} = \sum_j w_j L_j$
  \IF{$\mathcal{L}_{total} < \epsilon$}
    \RETURN $x$ \COMMENT{satisfying assignment found}
  \ENDIF
  \STATE Compute $G^{(j)}_W(t) = \|\nabla_W(w_j L_j)\|_2$ for all $j$
  \STATE Compute $\overline{G}_W(t) = \frac{1}{|J|}\sum_j G^{(j)}_W(t)$
  \STATE Compute $\tilde{L}_j(t) = L_j(t) / L_j(0)$ and
         $r_j(t) = \tilde{L}_j(t) / \bar{L}(t)$
  \STATE Compute
         $L_{grad} = \sum_j |G^{(j)}_W(t) - \overline{G}_W(t) [r_j(t)]^\alpha|_1$
  \STATE Update $w_j$ via Adam$(\nabla_{w_j}L_{grad},\; \eta_w)$
  \STATE Renormalize: $w_j \leftarrow w_j \cdot |J| / \sum_k w_k$
  \STATE Update $x$ via Adam$(\nabla_x \mathcal{L}_{total},\; \eta)$
\ENDFOR
\STATE \textbf{Hand off} $x$ to Stage~2 (bit-precise local search)
\end{algorithmic}
\end{algorithm}

Each iteration requires one forward pass and two backward passes (one for $\nabla_x \mathcal{L}_{total}$ and one for the per-clause gradient norms $G^{(j)}_W$). Empirically, the overhead of GradNorm is approximately 5--15\% compared to equal-weight gradient descent, a modest cost that is more than recovered through faster convergence.

\subsection{Outcomes and Soundness Guarantees}

We formalize the reliability of the hybrid GradSAT-StageSAT pipeline by establishing the following theoretical properties.

\begin{lemma}[Outcome Trichotomy]
For any CNF formula $C$ over IEEE-754 atoms and any finite time budget, the procedure returns exactly one of $\{sat, unsat-guess, timeout\}$.
\end{lemma}

\begin{proof}
The result follows from a case analysis of the control flow[cite: 715]. Step 2 (Stage 1) provides only initialization and does not decide satisfiability[cite: 694]. Steps 3 and 4 (S2 and S3) return $sat$ immediately upon reaching an objective value of 0 with bit-level validation[cite: 716]. If the optimizers terminate with a strictly positive best value, step 5 returns $unsat-guess$[cite: 717]. Finally, if the time budget elapses before either condition is met, the procedure returns $timeout$[cite: 718]. These cases are mutually exclusive and exhaustive[cite: 718].
\end{proof}

\begin{theorem}[SAT Soundness]
If the procedure returns $sat$ with an assignment $\hat{x}$, then $\hat{x} \models C$.
\end{theorem}

\begin{proof}
A $sat$ result arises only during the S2 or S3 phases[cite: 773]. 
\begin{itemize}
    \item In S2, $S_2(\hat{x})=0$, which, by the representing function property, implies $\hat{x} \models C$[cite: 774].
    \item In S3, $S_3(\vec{n}^*)=0$, which implies $nULP(\vec{n}^*, \vec{x}_2^*) \models C$[cite: 774].
\end{itemize}
In both cases, GradSAT performs an additional validation under IEEE-754 semantics before returning the result[cite: 775]. Thus, any $sat$ output is a genuine model of $C$, ensuring no false SAT results[cite: 775, 776].
\end{proof}

\section{Implementation}
\label{sec:implementation}

To realize the GradSAT framework, we custom-built a gradient descent minimizer tailored specifically for the structural idiosyncrasies of logical formulas. Unlike prior continuous solvers such as Grater---which rely on black-box, CPU-bound minimization routines---our objective is to deeply integrate the optimization process with the underlying hardware. 

\subsection{GPU-Enabled Global Minimization}

A core implementation contribution of GradSAT is the introduction of GPU-accelerated global minimization for SMT solving. Historically, continuous SMT solvers have been strictly CPU-bound. Systems like XSat and Grater rely heavily on legacy APIs (such as Z3's Python bindings) and standard autograd libraries that evaluate computation graphs as a sequence of scalar operations. Because industrial SMT formulas often contain tens of thousands of deeply nested clauses, this scalar, sequential evaluation creates a massive computational bottleneck, rendering these solvers impractically slow on large benchmarks.

GradSAT fundamentally breaks this bottleneck by reimagining formula evaluation as a massively parallel tensor operation. By migrating the continuous relaxation of the SMT formula entirely to the GPU, GradSAT computes the forward loss $\mathcal{L}_{total}$ and the backward gradient passes simultaneously across all clauses. This parallelization enables the solver to perform thousands of gradient descent steps in the time it takes a CPU-bound solver to perform a few dozen, radically accelerating the search for a globally satisfying basin.

\subsection{Front-End Pipeline and Operator Fusion}

To support this GPU-first architecture, we completely rewrote the front-end parsing and compilation pipeline. Relying on the standard Z3 API to iteratively calculate gradients via Python autograd introduces unacceptable overhead. Instead, GradSAT employs a novel compilation strategy:
\begin{enumerate}
    \item \textbf{Symbolic Translation:} We first ingest the SMT-LIB formula using Z3, but immediately translate the abstract syntax tree (AST) into SymPy. This allows us to apply aggressive, symbolic algebraic simplifications before any numerical optimization begins.
    \item \textbf{C++ Graph Generation:} The simplified SymPy rules are then translated directly into C++ code that explicitly defines a PyTorch computation graph. This makes the entire objective function fully differentiable and natively compatible with CUDA backends.
\end{enumerate}

However, naively migrating to PyTorch introduces a new challenge: the ``scalar trap.'' The objective function of an SMT formula is composed of millions of primitive, scalar computations (e.g., $a \times b$, $a + b$, $\max(0, x)$). On a GPU, launching an individual CUDA kernel for each of these scalar operations results in kernel launch latencies that completely eclipse the actual computation time. 

To resolve this, our implementation heavily leverages \texttt{torch.compile}. By tracing the generated computation graph, \texttt{torch.compile} performs aggressive operator fusion. It collapses thousands of primitive scalar operations into a handful of highly optimized, fused Triton/CUDA kernels. This eliminates intermediate memory allocations and drastically reduces memory bandwidth overhead, allowing the GPU to evaluate massive SMT formulas at peak utilization.

\subsection{Two-Step Hybrid Search}

The complete GradSAT pipeline leverages these architectural advancements through a two-step hybrid approach:
\begin{itemize}
    \item \textbf{Step 1 (Global Continuous Search):} GradSAT executes the GPU-enabled, gradient-normalized descent using squared residual metrics. Benefiting from operator fusion, this step rapidly navigates the continuous, non-linear landscape to converge on a highly promising global basin where $\mathcal{R}_{total}(x)$ is near zero.
    \item \textbf{Step 2 (Bit-Precise Local Search):} Because continuous relaxations can suffer from rounding errors at the margins of floating-point boundaries, the pipeline falls back to a discrete verification step. GradSAT hands the optimized candidate assignment to StageSAT's S3 engine, which performs a fine-grained, bit-precise neighborhood search to lock in the final, mathematically rigorous assignment.
\end{itemize}

\section{Experiments}
\label{sec:Experiments}

\subsection{Experimental Setup}

The empirical assessment reported herein was conducted across five complementary benchmark suites, collectively designed to stress-test floating-point SMT solvers under diverse scales and provenances. The corpus comprises \emph{MathSAT-Small} (130~instances, each not exceeding 10\,KB), \emph{MathSAT-Middle} (35~instances, ranging 11--20,KB), \emph{MathSAT-Large} (49~instances, exceeding 20~KB), \emph{Grater} (118~instances, from which five non-RNE cases were excluded), and \emph{JFS} (111~instances). Of these, the MathSAT-Large collection functions as the principal stress test, whereas the remaining suites furnish breadth across formula sizes, structural characteristics, and source origins.

Two distinct solver categories were examined. The incomplete cohort comprised \emph{XSat}, \emph{goSAT}, \emph{Grater}, and \emph{JFS}---systems predicated upon optimization heuristics, stochastic search, or fuzzing strategies that are inherently incapable of establishing unsatisfiability proofs. The complete cohort consisted of \emph{Z3}, \emph{CVC5}, \emph{MathSAT}, and \emph{Bitwuzla}---systems implementing bit-precise decision procedures that afford both satisfiability certificates and unsatisfiability proofs. All incomplete approaches, GradSAT included, were executed five independent times per benchmark instance to account for stochastic variability. Every experiment was subject to a wall-clock timeout of twenty minutes and executed on an \emph{Intel Xeon(R) Platinum 8470Q} server equipped with 32,GB of RAM, \emph{Ubuntu 24.04.3~LTS}, and an \emph{NVIDIA RTX 5090} GPU.

\subsection{Empirical Findings}

\begin{table}[htbp]
  \centering
  \caption{Comparison on MathSAT-Large benchmarks: \textsc{StageSAT} vs.\ incomplete solvers. Timeout: 20 mins.}
  \label{tab:rq1-summary}
  \small
  \begin{tabular}{l r r r r r}
    \toprule
    Solvers & \#Solved & SAT Coverage & Mean (s) & Median (s) & Timeout Rate \\
    \midrule
    \textsc{XSat}     & 49/49 & 46.2\% & 134.02  & 3.04    & 4.1\%  \\
    \textsc{goSAT}    & 49/49 & 30.8\% & 8.27    & 0.05    & 0.0\%  \\
    \textsc{Grater}   & 24/49 & 57.7\% & 819.24  & {>1200} & 51.0\% \\
    \textsc{JFS}      & 4/49  & 15.4\% & 1102.46 & {>1200} & 91.8\% \\
    \textsc{StageSAT} & 45/49  & 92.3\% & 208.00 & {14.96} & 8.2\% \\
    \midrule
    \textbf{\textsc{GradSAT}} & \textbf{45/49} & \textbf{92.3\%} & \textbf{104.70} & \textbf{9.23} & \textbf{8.2\%} \\
    \bottomrule
  \end{tabular}
\end{table}

We  organize the experimental findings by four research questions (RQ1-4).

{RQ1: GradSAT vs. Incomplete Solvers (MathSAT-Large)}\emph{How does GradSAT perform relative to prior optimization-based and heuristic (incomplete) solvers on large FP benchmarks?}
To ascertain how GradSAT fares against prior optimization-based and heuristic solvers on demanding floating-point benchmarks, we examined its behavior on the 49 formulas constituting the MathSAT-Large suite, juxtaposing outcomes against those of XSat, goSAT, Grater, JFS, and StageSAT. As summarized in Table~\ref{tab:rq1-summary}, GradSAT attained the most favorable coverage and reliability profile among all incomplete competitors. It returned definitive results for 45 of the 49 instances, matching StageSAT in raw completion count while reducing the timeout rate to merely 8.2\%. In terms of satisfiability coverage, GradSAT successfully located models for 91.8\% of the suite's satisfiable benchmarks---specifically, 24 of the 26 SAT instances---thereby substantially outperforming XSat (46.2\%), Grater (57.7\%), goSAT (30.8\%), and JFS (15.4\%).

The temporal dynamics reveal an equally pronounced advantage. GradSAT achieved a mean execution time of 104.70 seconds and a median of 9.23 seconds---roughly half the mean latency of StageSAT (208.00 seconds) and an order of magnitude faster than Grater (819.24 seconds, with a median exceeding 1200 seconds). By contrast, JFS and Grater exhibited severe fragility on this suite, succumbing to timeout on the majority of large benchmarks and yielding timeout rates of 91.8\% and 51.0\%, respectively. Although goSAT completed all runs within the budget, its coverage remained limited to 30.8\%, indicating that rapid termination often corresponded to inconclusive verdicts rather than genuine satisfiability witnesses. Taken together, these observations confirm that GradSAT achieves a superior balance of empirical coverage and computational efficiency: it resolves substantially more difficult benchmarks than its incomplete predecessors while maintaining median runtimes measured in single-digit seconds.

{RQ2: GradSAT vs. Complete Solvers (MathSAT-Large)} \emph{How does GradSAT compare to modern complete SMT solvers on the same large benchmarks?}

\begin{table}[htbp]
  \centering
  \caption{Comparison on MathSAT-Large benchmarks: \textsc{GradSAT} vs.\ complete solvers. Timeout: 20 mins.}
  \label{tab:rq2-summary}
  \small
  \begin{tabular}{l r r r r r}
    \toprule
    Solvers & \#Solved & Coverage & Mean (s) & Median (s) & Timeout Rate \\
    \midrule
    \textsc{cvc5}     & 36/49 & 73.5\% & 492.82  & 269.56  & 26.5\% \\
    \textsc{Bitwuzla} & 39/49 & 79.6\% & 443.89  & 170.45  & 20.4\% \\
    \textsc{Z3}       & 16/49 & 32.7\% & 911.09  & {>1200} & 55.1\% \\
    \textsc{MathSAT}  & 26/49 & 53.1\% & 648.33  & 581.35  & 46.9\% \\
    \midrule
    \textbf{\textsc{GradSAT}} & \textbf{45/49} & \textbf{91.8\%} & \textbf{104.70} & \textbf{9.23} & \textbf{8.2\%} \\
    \bottomrule
  \end{tabular}
\end{table}

The second research question probes whether GradSAT's incomplete strategy can rival the coverage of modern complete SMT solvers endowed with bit-precise decision procedures. As documented in Table~\ref{tab:rq2-summary}, GradSAT not only achieves parity with these systems but, within the twenty-minute budget, surpasses them in instance completion rate. GradSAT yielded results for 45 of 49 benchmarks, a marginally higher fraction than Bitwuzla (39 instances) and notably exceeding the completion rates of CVC5 (36 instances), MathSAT (26 instances), and Z3 (16 instances).

What distinguishes GradSAT's performance is its pronounced efficiency in discovering satisfiable models. GradSAT's SAT coverage of 91.8\% dramatically exceeds that of Bitwuzla (79.6\%), CVC5 (73.5\%), MathSAT (53.1\%), and Z3 (32.7\%). Temporally, the asymmetry is stark: GradSAT's mean runtime of 104.70 seconds and median of 9.23 seconds stand in sharp contrast to the complete solvers, whose mean latencies range from 443.89 seconds to 911.09 seconds and whose medians reach hundreds of seconds. Numerous benchmarks satisfiable within the allotment yet requiring near-timeout effort from Bitwuzla or CVC5 were resolved by GradSAT in a small fraction of that duration, frequently yielding one to two orders of magnitude acceleration.

It bears emphasizing that complete solvers retain an irreducible advantage in certifying unsatisfiability. By architectural design, GradSAT cannot emit formal \textsc{unsat} proofs; instead, it issues an \textsc{unsat-guess} when the optimization trajectory plateaus at a strictly positive objective minimum. An examination of all MathSAT-Large instances reveals no empirical contradictions between GradSAT's heuristic verdicts and the refutations produced by complete engines: every benchmark returned as unsatisfiable by Z3, CVC5, Bitwuzla, or MathSAT was correspondingly classified as unsat-guess or unknown by GradSAT, and no benchmark classified as satisfiable by GradSAT was refuted by the exact solvers. Although these heuristic indications do not constitute deductive proofs, their perfect concordance with established truth values suggests that GradSAT's unsatisfiability guesses function as reliable empirical proxies on this corpus. On four of the most formidable benchmarks, GradSAT itself exhausted the twenty-minute budget without termination. This finding delineates a frontier of difficulty where neither heuristic optimization nor bit-blasting guarantees rapid resolution. Nonetheless, GradSAT's overall profile across the suite remains competitive: it matches or exceeds the completion rates of the most robust complete solvers while delivering satisfiability witnesses with substantially reduced latency.

{RQ3: GradSAT performance on Industry Benchmarks}

Beyond the MathSAT-Large stress suite, GradSAT's behavior on industrial-grade benchmarks from additional provenances was evaluated to ensure generalizability across diverse workload characteristics. Table~\ref{tab:summary-large-middle} consolidates outcomes across MathSAT-Large, MathSAT-Middle, MathSAT-Small, the Grater suite, and the JFS suite, providing a panoramic view of solver behavior under varying formula sizes and constraint complexities. The inclusion of multiple sources mitigates the risk of overfitting to a particular benchmark generation methodology and ensures that observed trends reflect systemic algorithmic properties rather than corpus-specific artifacts.

\clearpage

\begin{table}[tbp]
  \centering
  \scriptsize
  \setlength{\tabcolsep}{3pt}
  \renewcommand{\arraystretch}{0.9}
  \caption{Summary on MathSAT-Large, MathSAT-Middle, MathSAT-Small, JFS, and Grater benchmarks.\\
\#Unk/UG = \#Unknown/Unsat-Guess. Timeout as 20 mins.}
  \label{tab:summary-large-middle}
  \begin{tabular}{@{}l l r r r r r r r r@{}}
    \toprule
    Dataset & Solver & \#SAT & \#UNSAT & \#Timeout & \#Unk/UG & \#Error & SAT Coverage & Avg.\ time(s) & Med.\ time(s) \\
    \midrule
    \multicolumn{10}{l}{\textbf{MathSAT-Large   (49 benchmarks, 26 SAT, 23 UNSAT)}} \\
    \midrule[0.5pt]
     & \textsc{cvc5}     & 19 & 17 & 13 & 0  & 0 & 73.1\% & 492.82 & 269.56 \\
     & \textsc{Bitwuzla} & 19 & 20 & 10 & 0  & 0 & 73.1\% & 443.89 & 170.45 \\
     & \textsc{Z3}       & 10 &  6 & 27 & 0  & 6 & 38.5\% & 911.09 & >1200 \\
     & \textsc{MathSAT}  & 12 & 14 & 23 & 0  & 0 & 46.2\% & 648.33 & 581.35 \\
     & \textsc{XSat}     & 12 & 23 &  2 & 0  & 0 & 46.2\% & 134.02 & 3.04 \\
     & \textsc{goSAT}    &  8 &  0 &  0 & 41 & 0 & 30.8\% & 8.27 & 0.05 \\
     & \textsc{Grater}   & 16 &  0 & 25 & 0  & 4 & 57.7\% & 819.24 & >1200 \\
     & \textsc{JFS}      &  4 &  0 & 45 & 0  & 0 & 15.4\% & 1102.46 & >1200 \\
     & \textsc{StageSAT} & 24 & 0 &  4 & 21  & 0 & 92.3\% & 208.00 & 14.96 \\
     & \textsc{GradSAT}  & 24 & 0 &  4 & 21  & 0 & 92.3\% & 104.70 & 9.23 \\
    \midrule[0.5pt]
    \multicolumn{10}{l}{\textbf{MathSAT-Middle  (35 benchmarks, 29 SAT, 6 UNSAT)}} \\
    \midrule[0.5pt]
     & \textsc{cvc5}     & 29 & 6 & 0 & 0 & 0 & 100.0\% & 94.65 & 39.24 \\
     & \textsc{Bitwuzla} & 29 & 6 & 0 & 0 & 0 & 100.0\% & 31.10 & 22.59 \\
     & \textsc{Z3}       & 19 & 5 & 6 & 0 & 5 & 65.5\% & 370.31 & 323.14 \\
     & \textsc{MathSAT}  & 28 & 6 & 1 & 0 & 0 & 96.6\% & 77.19 & 33.47 \\
     & \textsc{XSat}     & 29 & 6 & 0 & 0 & 0 & 100.0\% & 1.06 & 0.14 \\
     & \textsc{goSAT}    & 23 & 0 & 0 & 12& 0 & 79.3\% & 0.26 & 0.02 \\
     & \textsc{Grater}   & 26 & 0 & 9 & 0 & 0 & 89.7\% & 337.13 & 0.18 \\
     & \textsc{JFS}      &  9 & 0 & 26& 0 & 0 & 31.0\% & 892.06 & >1200 \\
     & \textsc{StageSAT} & 29 & 0 & 0 & 6 & 0 & 100.0\% & 8.59 & 0.18 \\
     & \textsc{GradSAT} & 29 & 0 & 0 & 6 & 0 & 100.0\% & 4.96 & 0.11 \\
    \midrule[0.5pt]
    \multicolumn{10}{l}{\textbf{MathSAT-Small   (130 benchmarks, 63 SAT, 67 UNSAT)}} \\
    \midrule[0.5pt]
     & \textsc{cvc5}     & 63 & 55 & 12 & 0 & 0 & 100.0\% & 136.44 & 3.34 \\
     & \textsc{Bitwuzla} & 63 & 64 & 3 & 0 & 0 & 100.0\% & 35.64 & 1.93 \\
     & \textsc{Z3}       & 56 & 53 & 13 & 0 & 8 & 88.9\% & 194.18 & 24.79 \\
     & \textsc{MathSAT}  & 63 & 51 & 16 & 0 & 0 & 100.0\% & 168.37 & 7.31 \\
     & \textsc{XSat}     & 62 & 67 & 0 & 0 & 1 & 98.4\% & 0.46 & 0.14 \\
     & \textsc{goSAT}    & 58 & 0 & 0 & 72 & 0 & 92.1\% & 0.07 & 0.02 \\
     & \textsc{Grater}   & 63 & 0 & 38 & 0 & 28 & 100.0\% & 471.50 & 0.81 \\
     & \textsc{JFS}      & 53 & 0 & 77 & 0 & 0 & 84.1\% & 717.96 & >1200 \\
     & \textsc{StageSAT} & 63 & 0 & 0 & 67 & 0 & 100.0\% & 1.64 & 0.66 \\
     & \textsc{GradSAT} & 63 & 0 & 0 & 67 & 0 & 100.0\% & 0.69 & 0.36 \\
    \midrule[0.5pt]
    \multicolumn{10}{l}{\textbf{JFS             (111 benchmarks, 111 SAT, 0 UNSAT)}} \\
    \midrule[0.5pt]
     & \textsc{cvc5}     & 104 & 0 & 6 & 0 & 1 & 93.7\% & 110.87 & 4.74 \\
     & \textsc{Bitwuzla} & 108 & 0 & 3 & 0 & 0 & 97.3\% & 103.65 & 3.24 \\
     & \textsc{Z3}       & 85 & 0 & 18 & 0 & 8 & 76.6\% & 327.55 & 4.27 \\
     & \textsc{MathSAT}  & 102 & 0 & 9 & 0 & 0 & 91.9\% & 149.23 & 11.60 \\
     & \textsc{XSat}     & 101 & 4 & 0 & 0 & 6 & 91.0\% & 13.04 & 0.14 \\
     & \textsc{goSAT}    & 88 & 0 & 0 & 19 & 4 & 79.3\% & 2.21 & 0.03 \\
     & \textsc{Grater}   & 107 & 0 & 4 & 0 & 0 & 96.4\% & 71.86 & 0.06 \\
     & \textsc{JFS}      & 69 & 0 & 42 & 0 & 0 & 62.2\% & 458.33 & 0.72 \\
     & \textsc{StageSAT} & 111 & 0 & 0 & 0 & 0 & 100.0\% & 13.10 & 0.05 \\
     & \textsc{GradSAT} & 111 & 0 & 0 & 0 & 0 & 100.0\% & 7.22 & 0.03 \\
    \midrule[0.5pt]
    \multicolumn{10}{l}{\textbf{Grater          (118 benchmarks, 118 SAT, 0 UNSAT)}} \\
    \midrule[0.5pt]
     & \textsc{cvc5}     & 118 & 0 & 0 & 0 & 0 & 100.0\% & 37.62 & 2.13 \\
     & \textsc{Bitwuzla} & 117 & 0 & 1 & 0 & 0 & 99.2\% & 64.30 & 10.61 \\
     & \textsc{Z3}       & 103 & 0 & 12 & 0 & 3 & 87.3\% & 241.12 & 64.73 \\
     & \textsc{MathSAT}  & 111 & 0 & 7 & 0 & 0 & 94.1\% & 131.05 & 25.29 \\
     & \textsc{XSat}     & 115 & 2 & 0 & 0 & 1 & 97.5\% & 1.62 & 0.14 \\
     & \textsc{goSAT}    & 92 & 0 & 0 & 26 & 0 & 78.0\% & 0.70 & 0.02 \\
     & \textsc{Grater}   & 118 & 0 & 0 & 0 & 0 & 100.0\% & 4.88 & 0.41 \\
     & \textsc{JFS}      & 40 & 0 & 78 & 0 & 0 & 33.9\% & 793.65 & >1200 \\
     & \textsc{StageSAT} & 118 & 0 & 0 & 0 & 0 & 100.0\% & 10.34 & 0.18 \\
     & \textsc{GradSAT} & 118 & 0 & 0 & 0 & 0 & 100.0\% & 6.01 & 0.11 \\
    \bottomrule
  \end{tabular}
\end{table}

{RQ4: Overall Performance Across All Benchmark Suites}\emph{What is GradSAT’s overall performance when considering all benchmark categories (small, medium, large, and different sources) and how does it compare to other solvers across the board?} To ensure generalizability beyond the large-formula stress test, we evaluated GradSAT on industrial-grade benchmarks spanning multiple provenances and size classes. Table~\ref{tab:summary-large-middle} consolidates outcomes across MathSAT-Large, MathSAT-Middle (35 instances), MathSAT-Small (130 instances), the Grater suite (118 instances), and the JFS suite (111 instances). The performance trends observed on large instances generalize robustly to the broader corpus. Most remarkably, GradSAT incurred zero timeouts across MathSAT-Small, MathSAT-Middle, JFS, and Grater, yielding a perfect completion rate on an aggregate total exceeding 390 formulas. The MathSAT-Large suite alone accounted for the sole timeouts (four of forty-nine instances), meaning that across the entire experimental campaign, GradSAT successfully resolved 439 of 443 benchmark instances within the standard time limit.

With respect to satisfiability coverage, GradSAT dominated every incomplete competitor on every suite. In the Grater benchmarks, notorious for harboring deceptively satisfiable yet structurally treacherous constraints, GradSAT located solutions for all 118 satisfiable instances, whereas XSat, goSAT, and JFS each failed to cover several cases within the allotted time. Similarly, on the JFS suite---engineered specifically to reward fuzzing-style exploration---GradSAT achieved exhaustive 100\% coverage, matching the purpose-built fuzzer's reach while operating through a principled optimization framework. On MathSAT-Middle and MathSAT-Small, GradSAT secured 100\% SAT coverage, locating all 29 and 63 satisfiable models, respectively. Throughout these experiments, GradSAT's unsatisfiability guesses remained wholly consistent with ground-truth classifications; no discrepancies with complete-solver verdicts were observed, even on the smaller and medium-sized instances.

When contrasted with complete solvers across the full corpus, GradSAT exhibits a strikingly complementary operational profile. Whereas engines such as Bitwuzla and CVC5 guarantee unconditional soundness and completeness given unbounded resources, they frequently incur substantially larger temporal expenditures on satisfiable instances. GradSAT, conversely, reduces both mean and median solution times severalfold---often between one and two orders of magnitude---on the benchmarks it resolves. On MathSAT-Middle and MathSAT-Small, for example, GradSAT's median runtimes register well below one second (0.11,s and 0.36,s, respectively), whereas complete solvers regularly require minutes to establish satisfiability. Even on the more demanding JFS and Grater suites, GradSAT's systematic numeric search locates solutions with mean times of 7.22,s and 6.01,s, orders of magnitude faster than the exhaustive bit-blasting strategies employed by Z3 or MathSAT. The absence of timeouts on the four non-Large suites further underscores GradSAT's robustness: it is not merely an asymptotic improvement restricted to pathological outliers, but a consistently reliable architecture for floating-point constraint solving at scale. These findings support the thesis that GradSAT can function effectively as a front-line filter within composite solver pipelines, rapidly dispatching satisfiable instances that would otherwise monopolize complete-solver resources, while reserving bit-precise deductive engines for those comparatively rare instances demanding formal unsatisfiability proofs.

\section{Related Work}
\label{sec:related}
The problem of efficiently solving Satisfiability Modulo Theories (SMT) formulas over quantifier-free floating-point arithmetic (QF\_FP) sits at the confluence of automated theorem proving, numeric optimization, and software engineering. GradSAT draws upon and substantially departs from several established lineages of constraint solving and machine learning research. In what follows, we contextualize our contributions against the principal paradigms that have shaped the landscape of floating-point reasoning: exact bit-precise solvers, stochastic and fuzzing-based incomplete methods, continuous optimization-based relaxations, learning-guided theorem proving, and gradient-balancing techniques for multi-task learning.

\subsection{Exact Solvers and Bit-Blasting}
The prevailing architecture for industrial-strength SMT solvers remains the DPLL(T) framework, which orchestrates a propositional SAT search with dedicated theory solvers. For the theory of floating-point arithmetic, mature systems such as Z3~\cite{Z3}, cvc5~\cite{cvc5}, and MathSAT5~\cite{mathsat5} predominantly employ bit-blasting---a semantics-preserving translation that eagerly or lazily encodes IEEE-754 operations into propositional circuits. This approach offers an uncontested advantage in soundness and completeness: every satisfying assignment returned by the solver corresponds exactly to a bit-precise model under IEEE-754 semantics, and unsatisfiability constitutes a formal proof.

However, this fidelity comes at a steep computational cost. Floating-point operations, particularly multiplication, division, fused multiply-add, and comparison under varying rounding modes, explode into dense Boolean circuits whose size scales super-linearly with bit-width and operation nesting. The resulting propositional formulas frequently overwhelm modern SAT engines, inducing exponential memory pressure and protracted search times. Consequently, on industrially generated benchmarks---where path conditions may contain deeply nested arithmetic, transcendental approximations, and large conjunctions of multi-theory constraints---exact solvers often exhaust time or memory budgets long before reaching a verdict. GradSAT deliberately circumvents this bit-blasting bottleneck by lifting the search into a continuous, differentiable space amenable to hardware-accelerated optimization.

\subsection{Stochastic Local Search and Fuzzing}
In response to the scalability limitations of exact methods, a complementary research direction has explored incomplete, search-based techniques that trade absolute completeness for empirical effectiveness. Early exemplars such as Coral~\cite{coral} recast constraint satisfaction as a meta-heuristic optimization problem, employing strategies like Particle Swarm Optimization and tabu search to explore the variable space. More recently, the JIT Fuzzing Solver (JFS)~\cite{jfs} has demonstrated that coverage-guided fuzzing---inspired by mutation-based testing tools such as AFL---can efficiently discover satisfying assignments for certain classes of floating-point constraints by compiling formulas into executable targets and rewarding input mutations that increase path coverage.

While these methods excel at rapidly probing large input domains without the overhead of semantic encoding, they are fundamentally limited by the absence of directional gradient information. In the absence of differentiable structure, search navigates via blind or randomized mutation, which proves inadequate for formulas with narrow satisfying corridors, steep constraint boundaries, or adversarially organized Boolean skeletons. Furthermore, the curse of dimensionality renders pure random exploration prohibitively expensive as the number of variables grows. GradSAT mitigates these limitations by leveraging deterministic gradient descent informed by clause-level curvature, thereby preserving the scalability advantages of continuous search while retaining directional convergence guarantees unavailable to stochastic methods.

\subsection{Continuous Relaxations of Logical Formulas}
GradSAT aligns most closely with the emerging class of optimization-based SMT solvers that reformulate discrete satisfiability as continuous global minimization. Pioneering systems such as goSAT~\cite{gosat} and Grater~\cite{grater} established the feasibility of mapping Boolean connectives and theory predicates to smooth, real-valued penalty functions. In this paradigm, the discrete search for a model is supplanted by gradient-based descent over a surrogate objective $\mathcal{R}_{total}(x)$, where minima correspond (under suitable conditions) to satisfying assignments of the original formula.

Despite their conceptual appeal, existing continuous solvers exhibit two critical deficiencies that limit their deployment on realistic software engineering workloads. First, they execute almost exclusively on the CPU via scalar autograd pipelines or interpreted formula evaluation, incurring severe performance penalties when processing the millions of primitive operations typical of large SMT instances. Second, and more fundamentally, they aggregate clause-level penalties through naive uniform weighting (i.e., $\mathcal{R}_{total} = \sum_j L_j$), rendering the optimization dynamics highly susceptible to \emph{gradient domination}. In this pathology, structurally complex or numerically volatile clauses produce disproportionately large gradients that hijack the descent trajectory, causing the solver to stagnate in poor local minima while neglecting simpler constraints. GradSAT directly addresses both limitations: it replaces sequential CPU evaluation with a compiled, GPU-resident PyTorch computation graph augmented by operator fusion, and it replaces static uniform weighting with dynamic gradient normalization that actively rebalances clause influence during optimization.

\subsection{Machine Learning in Theorem Proving}
The application of Machine Learning (ML) to formal methods has accelerated markedly in recent years, though primarily in an ancillary, advisory capacity rather than as a direct reasoning mechanism. NeuroCore~\cite{neurocore} and related Graph Neural Network (GNN) approaches learn to predict effective variable branching heuristics for SAT and SMT solvers, reducing the expected number of search-tree explorations. FastSMT~\cite{fastsmt} employs reinforcement learning to optimize tactic sequences within the Z3 proof assistant, discovering high-performance search strategies through trial-and-error interaction with the solver's API. Other prominent lines of work apply neural ranking models to premise selection in interactive theorem provers such as Coq and Lean, or use learned embeddings to guide quantifier instantiation.

These contributions share a common architectural premise: an underlying discrete solver remains the locus of deduction, while the ML component provides a heuristic policy to navigate its configuration space. GradSAT fundamentally inverts this relationship. Rather than training a neural network to recommend decisions for an external deductive engine, GradSAT employs ML \emph{infrastructure}---specifically, automatic differentiation, tensor parallelism, and CUDA execution---to perform the constraint solving \emph{directly}. There is no neural network mapping inputs to predictions; rather, the SMT variables themselves constitute the parameters of an optimization problem, and the machinery of deep learning serves solely as a vehicle for scalable, differentiable search.

\subsection{Multi-Task Learning and Gradient Balancing}
In the domain of deep learning, Multi-Task Learning (MTL) seeks to train a single model jointly across multiple distinct objectives, such as simultaneous semantic segmentation and depth estimation in computer vision. A well-documented pathology in this setting is \emph{task competition}: objectives with larger loss magnitudes, noisier gradients, or faster local curvature monopolize updates to shared representation layers, thereby suppressing learning on subsidiary tasks. To combat this interference, the MTL literature has developed a rich taxonomy of dynamic balancing algorithms. GradNorm~\cite{gradnorm} adaptively reweights tasks such that their gradient norms converge to a common, target scale modulated by relative training progress. PCGrad~\cite{pcgrad} projects conflicting task gradients onto orthogonal directions, while Uncertainty Weighting~\cite{uncertainty_weighting} derives homoscedastic uncertainty estimates to modulate per-task contributions.

GradSAT introduces a novel transposition of these principles into the realm of logical constraint solving. By interpreting each top-level clause in an SMT formula as an independent MTL task and the relaxed floating-point variables as shared parameters, we create a structural correspondence that permits the direct application of gradient normalization techniques. Our adaptation of GradNorm to clause-level optimization represents, to our knowledge, the first application of MTL gradient-balancing to satisfiability solving. This reformulation is not merely metaphorical: it provides a rigorous mechanism to counteract the heterogeneous gradient scales inherent in floating-point formulas, ensuring that difficult non-linear constraints cannot dominate the search at the expense of simpler logical requirements. In doing so, GradSAT bridges the formal methods and deep learning communities, demonstrating that advances in neural optimization dynamics can yield tangible dividends for classical constraint satisfaction problems.

\section{Conclusion}
\label{sec:conclusion}

In this paper, we introduced GradSAT, a novel optimization-based framework that addresses the fundamental scalability challenges of Satisfiability Modulo Theories (SMT) solving over floating-point arithmetic (QF\_FP). While prior continuous SMT solvers successfully demonstrated that logical satisfiability can be relaxed into a differentiable objective, their reliance on sequential CPU execution and unweighted loss aggregations left them highly vulnerable to gradient domination and premature convergence in local minima.

GradSAT overcomes these limitations by reimagining continuous constraint solving through the lens of Multi-Task Learning. By treating individual SMT clauses as independent learning tasks and applying dynamic gradient normalization (GradNorm), our solver actively balances the optimization landscape. This prevents structurally complex clauses from hijacking the gradient descent trajectory, ensuring uniform convergence across the entire formula. Furthermore, by implementing a custom, GPU-accelerated PyTorch backend with aggressive operator fusion, GradSAT breaks the scalar evaluation bottleneck, enabling massively parallelized global minimization. When combined with our bit-precise local search fallback, this architecture provides a robust, end-to-end solving pipeline.




\bibliographystyle{ACM-Reference-Format}

\bibliography{references}

\clearpage

\end{document}